\documentclass{article}

\usepackage{microtype}
\usepackage{graphicx}
\usepackage{subfigure}
\usepackage{booktabs}
\usepackage[utf8]{inputenc}
\usepackage{amsmath}
\usepackage{amssymb}
\usepackage{amsthm}
\usepackage{xcolor}
\usepackage[normalem]{ulem}
\usepackage{wrapfig}
\usepackage{mathrsfs}
\usepackage{xspace}
\usepackage[font=small,labelfont=bf]{caption}
\newcommand{\citep}[1]{\cite{#1}}

\newtheorem{innercustomthm}{Theorem}
\newenvironment{customthm}[1]
  {\renewcommand\theinnercustomthm{#1}\innercustomthm}
  {\endinnercustomthm}

\newtheorem{innercustomcor}{Corollary}
\newenvironment{customcor}[1]
  {\renewcommand\theinnercustomcor{#1}\innercustomcor}
  {\endinnercustomcor}

\usepackage[pdftex,colorlinks,linkcolor=blue,citecolor=blue,filecolor=blue,urlcolor=blue]{hyperref}
\usepackage[letterpaper,hmargin=1.0in,vmargin=1.0in]{geometry}
\usepackage{nicefrac}

\newtheorem{theo}{Theorem}[section]
\newtheorem{cor}[theo]{Corollary}

\newtheorem{lemma}[theo]{Lemma}

\newtheorem{claim}[theo]{Claim}

\theoremstyle{remark}

\theoremstyle{definition}

\newenvironment{proofsketch}{%
  \proof}{\endproof}

\usepackage{bm}
\newcommand{\bW}{\bm{W}}
\newcommand{\bP}{\bm{P}}
\newcommand{\bR}{\bm{R}}
\newcommand{\bX}{\bm{X}}

\newcommand{\bD}{\bm{D}}
\newcommand{\bE}{\bm{E}}

\newcommand{\bb}{\bm{b}}
\newcommand{\bI}{\bm{I}}
\newcommand{\bx}{\bm{x}}

\newcommand{\br}{\bm{r}}
\newcommand{\bB}{\bm{B}}
\newcommand{\bA}{\bm{A}}

\newcommand{\bq}{\bm{q}}
\newcommand{\one}{\text{\textbf{1}}}
\newcommand{\bdelta}{\bm{\delta}}

\newcommand{\softmax}{\text{softmax}}
\newcommand{\lnorm}[1]{\text{LN}\hspace{-1pt}\left(#1\right)}

\newcommand{\res}{\text{res}}

\newcommand{\T}{\top}

\newcommand{\din}{d_{\textit{in}}}

\newcommand{\dv}{d_{\textit{v}}}
\newcommand{\dqk}{d_{\textit{qk}}}

\newcommand{\att}{\text{SA}\xspace}
\newcommand{\attnet}{\text{SAN}\xspace}
\newcommand{\attnets}{\text{SANs}\xspace}

\title{Attention is not \textit{all} you need: \\pure attention loses rank doubly exponentially with depth}

\author{
  Yihe Dong\\
  Google \\
  \href{mailto:yihed@google.com}{\color{black}{yihed@google.com}}
  \and
  Jean-Baptiste Cordonnier \\
  EPFL \\
  \href{mailto:jean-baptiste.cordonnier@epfl.ch}{\color{black}{jean-baptiste.cordonnier@epfl.ch}}
  \and
  Andreas Loukas \\
  EPFL \\
  \href{mailto:andreas.loukas@epfl.ch}{\color{black}{andreas.loukas@epfl.ch}}
}

\begin{document}
\date{}
\maketitle

\begin{abstract}
Attention-based architectures have become ubiquitous in machine learning.
Yet our understanding of the reasons for their effectiveness remains limited.
This work proposes a new way to understand self-attention networks: we show that their output can be decomposed into a sum of smaller terms, each involving the operation of a sequence of attention heads across layers.
Using this decomposition, we prove that self-attention possesses a strong inductive bias towards ``token uniformity". Specifically, without skip connections or multi-layer perceptrons (MLPs), the output converges doubly exponentially to a rank-1 matrix. On the other hand, skip connections and MLPs stop the output from degeneration. Our experiments verify the identified convergence phenomena on different variants of standard transformer architectures\footnote{Our code is publicly available at \url{https://github.com/twistedcubic/attention-rank-collapse}}.
\end{abstract}

\section{Introduction}

The attention mechanism \cite{bahdanau2015} was initially developed to better learn long-range sequential knowledge, and found effective use in transformer networks \cite{vaswaniTransformer}. Since then, attention-based architectures have permeated across data domains machine learning applications, such as in natural language processing \cite{devlin2018bert, decomposableAtt}, speech recognition \cite{luoSpeech2020}, and computer vision \cite{ramachandranCV2019, belloCV2019}. As such, it is vital to develop tools to understand the inner workings of transformers and attention in general, both to shed light on existing models, and to design more effective future models.

This work provides new insights about the operation and inductive bias of networks built by stacking multiple self-attention layers. %
Surprisingly, we find that \textit{pure} self-attention networks (\attnets), i.e., transformers with skip connections and multi-layer perceptrons (MLPs)  disabled, lose expressive power \textit{doubly exponentially} with respect to network depth. More specifically, we prove that the output converges with a cubic rate to a rank one matrix that has identical rows. %
While we derive the convergence bounds in part by using properties of stochastic matrices, our results go beyond what one would expect based on standard results. In particular, by leveraging the cascading effects of specifically stacking self-attention modules, we show exponentially faster convergence than what standard theory prescribes. Furthermore, while previous studies have considered the rank of individial self-attention matrices \cite{linformerWang, kernelTransformer,jbCollaborate}, our results are the first to address conditions under which the \textit{entire} network converges to rank \textit{one}.

This raises the question, why do transformers work?
Our analysis indicates that skip connections play a key role in mitigating rank collapse, and MLPs can slow down the convergence by increasing their Lipschitz constant. We characterize these counteracting forces by proving upper and lower bounds of this convergence behavior under \attnet architectural variants that resemble transformers. Our results reveal a previously unknown vital utility of skip connections, beyond facilitating optimization and gradient flow \cite{he2016deep, balduzzi2018shattered}.

In the process, we develop a new \textit{path decomposition} to study self-attention networks. %
Namely, we decompose a \attnet into a linear combination of weakly-interdependent \textit{paths}, where each `path' corresponds to a deep single-head \attnet. %
Intuitively, one can view the self-attention heads in each layer of the original network as different gateways, and a path follows a sequence of gateway choices, one gateway per layer (Figure~\ref{fig-path-diagram}).
Coupled with the rank collapse analysis, our results suggest that deep \attnets with skip connections %
behave like an ensemble of weakly-dependent shallow networks.

Our main contributions are as follows:
(1) We present a systematic study of building blocks of the transformer, revealing opposing impacts between self-attention and the counteracting forces: skip connections and MLP, in contributing and preventing a \textit{rank collapse} in transformers. As a corollary, this reveals a previously unknown vital effect of skip connections beyond facilitating optimization.
(2) We propose a new method for analyzing \attnets via a \textit{path decomposition}, revealing \attnets as an ensemble of shallow networks. %
(3) We verify our theory with experiments on common transformer architectures. %

\paragraph{Notation.}  In the sequel, bold-face lower/upper-case letters denote vectors and matrices, respectively. We denote the $\ell_1, \ell_\infty$-composite norm of a matrix $\bX$ as $\| \bX\|_{1,\infty} = \sqrt{\| \bX \|_1 \| \bX \|_\infty}$. We note that $\ell_{1,\infty}$ is not a proper norm as it does not satisfy the triangle inequality, though it is absolutely homogeneous and positive definite.  We also use the shorthand notation $[H] = (1,\cdots,H)$.

\section{Attention doubly exponentially loses rank}

We start by studying self-attention networks (\attnets) built exclusively out of multi-head self-attention layers. We prove that \attnets converge exponentially (with depth) to a rank-1 matrix that makes all tokens identical.

Our analysis in \S\ref{sec-path-dec} relies on an unconventional way to express the output of a multi-head \attnet as a sum of single-head networks. We refer to the latter as \textit{paths}, where each path is denoted by a sequence of attention heads %
(see Figure~\ref{fig-path-diagram}). A proof sketch of why rank collapse occurs is given in~\S\ref{sec-exp-path}, whereas the main rank collapse result is presented in~\S\ref{sec-exp-an}.

\begin{figure*}[t]
     \centering
     \includegraphics[width=.65\linewidth]{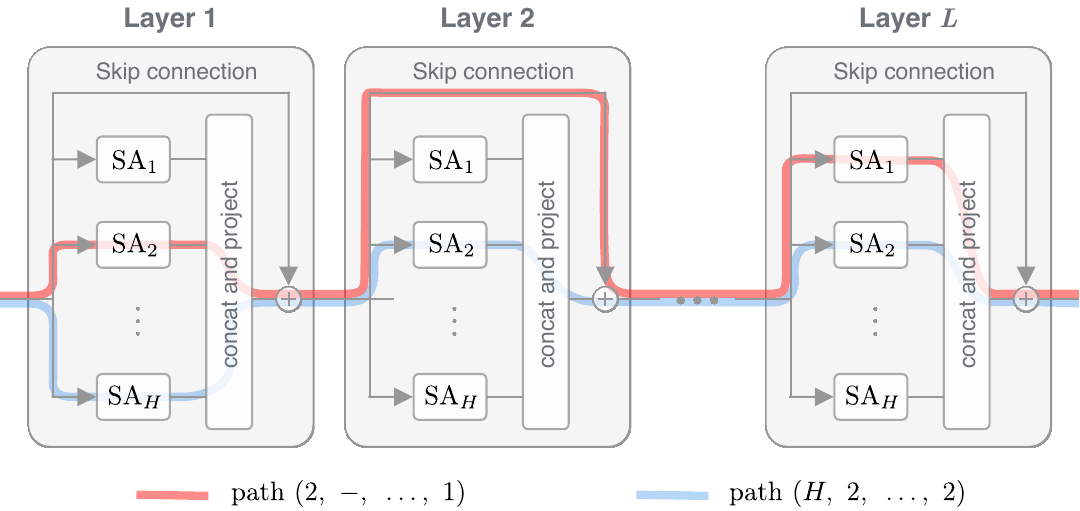}
         \caption{Two paths in a deep Self-Attention Network (\attnet{}) with $H$ heads and $L$ layers. At each layer, a path can go through one of the heads or bypass the layer. Adding an MLP block after each attention layer forms the transformer architecture.}
\label{fig-path-diagram}
\end{figure*} %

\subsection{The path decomposition argument}
\label{sec-path-dec}

Let $\bX$ be a $n \times \din$ input matrix consisting of $n$ tokens.
A \attnet is built out of $L$ multi-head self-attention layers, each having $H$ heads.
The output of the $h$-th self-attention head can be written as
\begin{align*}
    \att_h(\bX) = \bP_h \bX \bW_{V,h} + \one \bb_{V,h}^\T.
\end{align*}
Above, $\bW_{V,h}$ is a $ \din \times \dv$ value weight matrix and the $n \times n$ \textit{row-stochastic} matrix $\bP_h$ is given by
\begin{align*}
    \bP_h &= \softmax\big(\dqk^{-\frac{1}{2}}(\bX \bW_{Q,h} + \one \bb_{Q,h}^\T) (\bX \bW_{K,h} + \one\bb_{K,h}^\T )^\T\big) \\
    &= \softmax (\dqk^{-\frac{1}{2}}(\bX \bW_{QK,h} \bX^\T + \one \bb_{Q,h}^\T \bW_{K,h}^\T \bX^\T)),
\end{align*}
where (1) the key and query weight matrices $\bW_{K,h}$ and $\bW_{Q,h}$ are of size $\din \times \dqk$, (2) $\bW_{QK,h}= \bW_{Q,h} \bW_{K,h}^\T$, and (3) the softmax operates independently on each row of its input. We obtain the final equation by noting that softmax is shift-invariant and disregarding terms that provide a constant contribution across rows~\cite{jbCollaborate}.

The output of each \attnet layer is formed by concatenating the individual outputs of all $H$ attention heads (along the last dimension) and linearly projecting them onto a subspace of appropriate size:
\begin{align*}
    \att(\bX) &= \one [\bb_{O,1}^\T, \cdots, \bb_{O,H}^\T] +
    [\att_1(\bX), \cdots, \att_{H}(\bX)] \ [\bW_{O,1}^\T, \cdots, \bW_{O,H}^\T]^\T  \\
    &= \sum_{h \in [H]} \bP_h \bX \bW_{h} + \one \bb_{O}^\T,
\end{align*}
where we set $\bW_{h} = \bW_{V,h} \, \bW_{O,h}^\T$ and $\bb_{O} = \sum_{h} \bb_{O,h}$.

Let $\bX^{l}$ be the output of the $l$-th layer and fix $\bX^{0} = \bX$. As is common practice, we let all layers consist of the same number of heads.

Excluding biases $\one \bb_{O,h}^\T$, the \attnet output is given by
\begin{align*}
    \bX^L
    &= \sum_{h \in [H]} \bP_h^{L} \bX^{L-1} \bW_h^{L}  \nonumber \\
    &= \sum_{h \in [H]} \bP_h^{L} \bigg( \sum_{h' \in [H]} \bP_{h'}^{L-1} \bX^{L-2} \bW_{h'}^{L-1} \bigg) \, \bW_h^{L}
    = \hspace{-3mm} \sum_{h_L,h_{L-1} \in [H]^2} \hspace{-3mm} \bP_{h_L}^{L} \bP_{h_{L-1}}^{L-1} \bX^{L-2} \bW_{h_{L-1}}^{L-1} \bW_{h_L}^{L},     \nonumber
\end{align*}
which, after unrolling the recursion backwards, yields:
\begin{align*}
 \bX^L
    &= \hspace{-3mm} \sum_{h_1, \ldots, h_L \in [H]^L} \hspace{-3mm} (\bP_{h_L}^{L} \cdots \bP_{h_1}^{1}) \, \bX \, (\bW_{h_1}^{1} \cdots \bW_{h_L}^{L}).
\end{align*}
The above equations have a clear interpretation if we think of the \attnet as a directed acyclic graph, with nodes corresponding to self-attention heads and directed edge connecting heads of consecutive layers.

We formalize this intuition in the following:

\begin{theo}[Path decomposition of \attnet]
The output of a depth $L$ self-attention network with $H$ heads per layer (including biases and skip connections) is given by
\begin{align}
    \attnet(\bX) = \sum_{\textit{path} \in [H]^L} \bP_\textit{path} \, \bX \, \bW_{\textit{path}} + \one \bb^\T,
    \label{eq-path-simple}
\end{align}
where
$\bP_\textit{path} = \bP_{h_L}^{L} \cdots \bP_{h_1}^{1}$ is an input-dependent stochastic matrix, whereas $\bW_{\textit{path}} = \bW_{h_1}^{1} \cdots \bW_{h_L}^{L}$ and $\bb$ do not depend on the input.
\end{theo}

\begin{proof} The proof follows from the fact that the set of row-stochastic matrices is closed under multiplication (i.e., $\bP_{h_L}^{L} \cdots \bP_{h_i}^{i}$ is row-stochastic) and, moreover, for any row-stochastic matrix $\bP$, we have $\bP \one =  \one$.
\end{proof}

Each term in \eqref{eq-path-simple} describes a path of length $L$ across heads of different layers
$$
    \textit{path} = (h_1, \ldots, h_L), \ \text{ where } \ h_l\in (0,1,\ldots, H)
$$
and there are a total of $H^L$ such paths without skip connections.

The path decomposition thus describes the action of a multi-head \attnet as the combination of simpler single-head networks. To gain intuition on path interdependence, it helps to split the operations performed into two types: those that act across tokens (multiplication from left) and those that apply independently on each token (multiplication from right).
As seen, though paths can interact through token mixing (since $P_\textit{path}$ matrices jointly depend on $X$), token-wise operations are independent.
We can also notice that biases are not particularly meaningful: their total contribution amounts to the single term $\one \bb^\T$ independently of the number of layers or heads used.

In the following we show that each path converges rapidly (as a function of length) to a rank-1 matrix with identical rows. Interestingly, this convergence is so dominant that adding more layers to the \attnet does not help: though the number of paths is increased exponentially, each path degenerates doubly exponentially, leading also to a rank-1 output.

\subsection{Convergence of single-head \attnet}
\label{sec-exp-path}

Before tackling the full \attnet, it is instructive to consider the behavior of each path separately.
We examine, in particular, how the residual
$$
    \res(\bX) = \bX - \one \bx^\T, \ \text{ where } \  \bx = \text{argmin}_{\bx} \| \bX - \one \bx^\T\|
$$
changes during the forward pass.

As the following result shows, the residual norm converges to zero surprisingly quickly (doubly exponentially with a cubic rate):

\begin{theo}[Simplified]
\label{thm-cubic-conv}
For any single-head \attnet consisting of $L$ layers with $\|\bW_{QK}^l\|_1 \|\bW_{V}^{l}\|_{1,\infty} \leq \beta$ and for a term $\gamma$ that depends on the attention entries, we have that
\begin{align} \label{mhsa-conv_rate}
   \|\res(\attnet(\bX))\|_{1,\infty} \leq \bigg(\frac{4 \gamma \beta }{\sqrt{\dqk}}\bigg)^{\frac{3^L-1}{2}} \, \|\res(\bX) \|^{3^L}_{1,\infty},
\end{align}
which amounts to a doubly exponential convergence to a rank-1 matrix.
\end{theo}
For the full theorem, we refer the reader to the Appendix.

Note that the bound in Eq~\ref{mhsa-conv_rate} guarantees $\|\res(\attnet(\bX))\|_{1,\infty}$ convergence for all inputs of small residual whenever $4\gamma\beta < \sqrt{\dqk}$. In practice, our experiments imply that the region for convergence can be much greater.

The identified cubic rate of convergence is significantly faster than what would  be expected when analyzing products of stochastic matrices (linear rate).
As a rule of thumb, to achieve a decline of three orders of magnitude, say from 1000 to 1, one could expect a linear rate of convergence to require roughly a dozen iterations, whereas a cubic rate can do so in just two or three iterations.
The reason why we get a cubic rate is that the rank of attention matrices depends also on the rank of the input. As we show, the self-attention heads mix tokens faster when formed from a low-rank matrix. This phenomenon becomes stronger as we build deeper \attnets, leading to a cascading effect.

We provide a proof sketch bellow. Detailed proofs can be found in the Appendix.

\begin{proofsketch}
To analyze how the formation of $\bP_h$ is affected by the rank of the input, we start by writing $\bX = \one \bx^\T + \bR$ for $\bR = \res(\bX)$ and expanding the attention matrix accordingly:
\begin{align*}
    \bX \bW_{QK} \bX^\T
    = \left(\one \bx^\T + \bR\right) \bW_{QK} \left(\one \bx^\T + \bR\right)^\T
\end{align*}
Invoking once more the shift-invariance property of the softmax, the above can be simplified to
\begin{align*}
    \bP_h
    &= \softmax\left(\bR \frac{\bW_{QK}}{\sqrt{\dqk}} \bR^\T + \one \br^\T\right),
\end{align*}
for some appropriate $\br$. %
Observe that if the matrix within the softmax was $\one \br^\top$, then $\bP_h$ would also degenerate to a rank-1 matrix:
$
     \softmax( \one \br^\T) = \one \, \bq^\T
$
and the convergence would happen instantly.

The proof builds on this observation by showing that if $\bE = \bR \frac{\bW_{QK}}{\sqrt{\dqk}} \bR^\T$ is small then $\bP_h$ is almost rank-1:
$$
    \|\bP_h - \one \bq^\T\| \leq 2 \, \| \bD \, \one \, \bq^\T \|,
$$
where $\bD$ is diagonal and $\bD_{ii} = \max_{j} |\bdelta_i^\T \bE (\bdelta_j - \bdelta_{j'}) |$.
Thus, we have
\begin{align*}
    \bP_h \bX
    &= \bP_h (\one \bx^\T + \bR)  = \one \bx^\T + \softmax( \one \br^\T + \bE) \bR
\end{align*}
and, moreover,
$
    \| \res(\bP_h \bX)\| \leq 2 \, \| \bD \one \, \bq^\T \bR\|.
$
The proof concludes by bounding the above term and applying the argument recursively over successive layers.
\end{proofsketch}

\subsection{Exponential convergence for attention networks}
\label{sec-exp-an}

We now move on to analyse the convergence of \attnets with \textit{multiple} heads per layer.

Our main result is as follows:

\begin{theo}[Simplified] \label{thm-cubic-conv-mhsa}
Consider a depth-$L$ and width-$H$ self-attention network without skip connections. Suppose that $\| \bW_{QK,h}^l \|_1 \| \bW_{h}^l \|_{1, \infty} \leq \beta$ for all heads $h \in [H]$ and layers $l \in [L]$, and let $\gamma$ be a term that depends on the attention entries. We have
\begin{align} \label{eq-no-skip-conv-rate}
   \|\res(\attnet(\bX))\|_{1,\infty} \leq \left(\frac{4 \, \gamma \,\beta \, H}{\sqrt{\dqk}}\right)^{\frac{3^L-1}{2}} \, \|\res(\bX) \|^{3^L}_{1,\infty}, \notag
\end{align}
which amounts to a doubly exponential rate of convergence.
\end{theo}

The bound guarantees convergence of $\attnet(\bX)$ to rank one when $4 \gamma \beta H < \sqrt{\dqk}$. Our experiments show that this is a rather pessimistic estimate, as, in practice, we observe widespread convergence of output to rank-1.

\textbf{Remark 1. Implications for Xformers.} There has been a surge of architectural variants -- that we collectively refer to as Xformers -- aimed to improve the vanilla transformer \citep{vaswaniTransformer} by reducing the quadratic self-attention complexity. The rank collapse result of Theorem~\ref{thm-cubic-conv-mhsa} carries interesting implications for these architectures. One such variant relies on low-rank or kernel-based approximations to the full attention matrix \cite{kernelTransformer,linformerWang, kernelTransformerChoromanski}, in which case the paths likely converge even faster to rank one due to the imposed low-rankedness.
Another variant only computes a subset of the attention matrix entries using particular patterns \cite{bigbirdZaheer,sparseTransformer}, such as random patterns, in which case one expects the paths to converge more slowly, as randomization tends to increase the rank of the output. %

\section{Mechanisms that counteract rank collapse} \label{sec-counter-conv}

Our findings raise a pertinent question---\emph{why do attention-based networks work in practice if attention degenerates to a rank-1 matrix doubly exponentially with depth?} Aiming to obtain a deeper understanding, we focus on the transformer architecture \cite{vaswaniTransformer} and expand our analysis by incorporating the three important components of transformers that \attnets lack: \textit{skip connections}, \textit{multi-layer perceptrons}, and \textit{layer normalization}.

We adopt a methodical approach where the modifications to the \attnet architecture are introduced one at a time. For each case, we re-derive the convergence bounds and discuss the observed effect.

\subsection{Skip connections are crucial}

A simple modification to the path decomposition argument for \attnet suffices to take into account skip connections. Specifically, we indicate the event that a path has skipped a layer by setting $h=0$ on the corresponding notation:
\begin{align*}
    \bX^L
    &= \sum_{h \in [H]\cup \{0\}} \bP_h^{L} \bX^{L-1} \bW_h^{L}  \nonumber \\
    &= \dots \nonumber \\
    &= \hspace{-3mm} \sum_{h_1, \ldots, h_L \in ([H]\cup \{0\})^L} \hspace{-3mm} (\bP_{h_L}^{L} \cdots \bP_{h_1}^{1}) \, \bX \, (\bW_{h_1}^{1} \cdots \bW_{h_L}^{L}),
\end{align*}
where we have fixed $\bP_0=\bI$ and $\bW_0=\bI$.

As observed, skip connections dramatically diversify the path distribution.
Denote by $\mathcal{P}_l$ the set of paths of length $l$. With skip connections enabled, we have
$$
    |\mathcal{P}_l| = {L\choose l} \, H^l
$$
paths of length $l$ (whereas before we had only length $L$ paths).
We hypothesize that it is the presence of short paths that stops \attnet from degenerating to rank-1.

While we can derive an upper bound for the residual similar to above (which we do in the Appendix for completeness) such an upper bound is vacuously large.
Indeed, %
it is more informative to have a \textit{lower} bound on the residual, to align with practice, where \attnets with skip connections do not suffer rank collapse. We present the following simple lower bound:

\begin{claim}
Consider a depth-$L$ and width-$H$ self-attention network with skip connections. There exist infinitely many parameterizations for which $\|\res(\bX^L)\| \ge \| \res(\bX) \|$. The preceeding holds even for $L\to \infty$ and $\beta$ arbitrarily small.
\end{claim}

The proof is elementary: by the path decomposition, there is always a path that skips all layers, i.e. the path with length 0, preserving the residual. It then follows that, for any parametrization that renders the contribution of the \attnet layers orthogonal to the input, we will have $\|\res(\bX^L)\| \ge \| \res(\bX) \|$. A simple example of such a parametrization can be recovered by setting $\bW_V^l = 0$ for every $l \in [L]$, in which case $\|\res(\bX^L)\| = \| \res(\bX) \|$.

A tight lower bound to the residual in the presence of skip connections is highly nontrivial, and we pose it as an open challenge to the community.

\textbf{Remark 2. \attnets as ensembles of shallow networks.} %
It can be deduced from Theorem~\ref{thm-cubic-conv-mhsa} that the \attnets with skip connections enabled heavily rely on short paths (since the residual rapidly declines as the path length becomes larger). In other words, \attnets behave like ensembles of shallow single-head self-attention networks. The phenomenon was previously identified for ResNets~\cite{veit2016residual} (though the latter study didn't study the rank-collapse phenomenon). Here, the components of this ensemble are inter-dependent, as each attention head participates in many paths of different lengths. Experimental results in \S\ref{sec-experiments} support this implication. The supplementary material also provides a study of the paths distribution across several common architectures.

\subsection{Multi-layer perceptrons (MLP) help}
We now study how using an MLP affects the residual.
In particular, we focus on \attnets with layers written as
\begin{align*}
    \bX^{l+1}
    &= f_l\left( \sum_{h \in [H]} \bP_h \bX^{l} \bW_{h}\right).
\end{align*}
Note that, to keep the notation compact, we use $f_l$ to denote both the MLP as well as the output bias.

In our subsequent analysis, we use $\lambda_{l, 1,\infty}$ to denote the Lipschitz constant of $f_l$ with respect to $\ell_{1,\infty}$ norm. Note that, though finding the exact constant can be NP-hard even for shallow MLPs~\cite{scaman2018lipschitz}, since $f_l$ comprises of linear transformations with Lipschitz nonlinearities, $f_l$ is generally Lipschitz. %

\begin{cor}[Simplified] \label{cor-res-mlp}
Consider a depth-$L$ and width-$H$ \attnet with MLP. Suppose that $\|\bW_{QK,h}^l \|_1 \| \bW_{h}^{l}\|_{1,\infty}\leq \beta$ for all $h \in [H]$ and $l \in [L]$, let $\gamma$ be a term that depends on the attention entries, and fix $\lambda_{l, 1,\infty} \leq \lambda$. We have that
\begin{align}
    \| \res(\bX^L)\|_{1,\infty}
    &\leq \left( \frac{4 \, \gamma \, \beta \, H\, \lambda} { \sqrt{\dqk}}\right)^{\frac{3^L-1}{2}} \, \|\res(\bX)\|_{1,\infty}^{3^L},
\end{align}
which amounts to a doubly exponential rate of convergence. 
\end{cor}

As seen, though the effect of MLP is less drastic than that of skip connections, the convergence rate in Cor~\ref{cor-res-mlp} can be controlled by the Lipschitz constants $\lambda_{f,1,\infty}$ of the MLPs: the more powerful the MLPs are the slower the convergence becomes. This reveals a tug-of-war between the self-attention layers and the the MLPs, which due to their nonlinearity can increase the rank. \S\ref{sec-experiments} shows that indeed MLPs counteract convergence in experiments.

We should emphasize that using MLPs to counteract the rank-collapse is not without drawbacks: While increasing the Lipschitz constants slows down residual convergence, it also renders the model less robust and more sensitive to input perturbations~\cite{cranko2018lipschitz}. Larger Lipschitz constants may also pose greater challenges to optimization, as they lead to larger gradient variance.

\subsection{Layer normalization plays no role}

Layer normalization is accomplished by rescaling and shifting the input across the feature dimension:
\begin{align*}
    \lnorm{\att(\bX)}
    &= \lnorm{\sum_{h \in [H]} \bP_h \bX \bW_{h} + \one \bb_{O}^\T}
    = \bigg(\sum_{h \in [H]} \bP_h \bX \bW_{h} + \one \bb_{O}^\T  - \one \bb_{LN}^\T \bigg) \bD_{LN}^{-1},
\end{align*}
where $\bb_{\text{LN}}$ is the mean of each column $\att(\bX)$ and $\bD_{LN}$ is a diagonal matrix with entries corresponding to the (possibly scaled or shifted) standard deviation of each column $\att(\bX)$.

By setting $\tilde{\bW}_h = \bW_{h}\bD_{LN}^{-1} $ and $\tilde{\bb}_{O} = \bb_{O} - \bb_{LN}$, the above is re-written as
\begin{align*}
    \lnorm{\att(\bX)}
    &= \sum_{h \in [H]} \bP_h \bX \tilde{\bW}_{h} + \one \tilde{\bb}_{O}^\T,
\end{align*}
which is identical to the equation before layer normalization was applied, though now $\tilde{\bW}_{h}$ and $\tilde{\bb}_{O}$ are input dependent. %
Since right multiplication cannot increase the rank of a matrix,
we conclude that layer normalization does not mitigate the rank collapse.

\section{Experiments}
\label{sec-experiments}

Our experiments first test the rank collapse phenomenon
in several well-known transformers architectures (\S\ref{sec-exp-collapse}).
We also visually illustrate the inductive bias of some architectural variants of transformers with a toy example in \S\ref{sec-exp-circle} and test the paths effectiveness with respect to length in \S\ref{ssec:exp-paths}. Additional results can be found in the Appendix.

\begin{figure*}[t]
\centering
\begin{minipage}[b]{0.3\linewidth}
\centering
\includegraphics[width=1\linewidth]{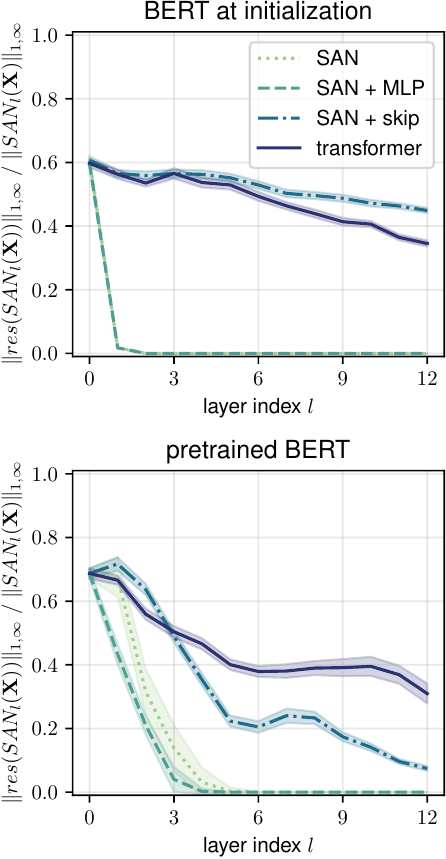}
\caption*{(a) Bert}
\end{minipage}
\hspace{3mm}
\begin{minipage}[b]{0.3\linewidth}
\centering
\includegraphics[width=1\linewidth]{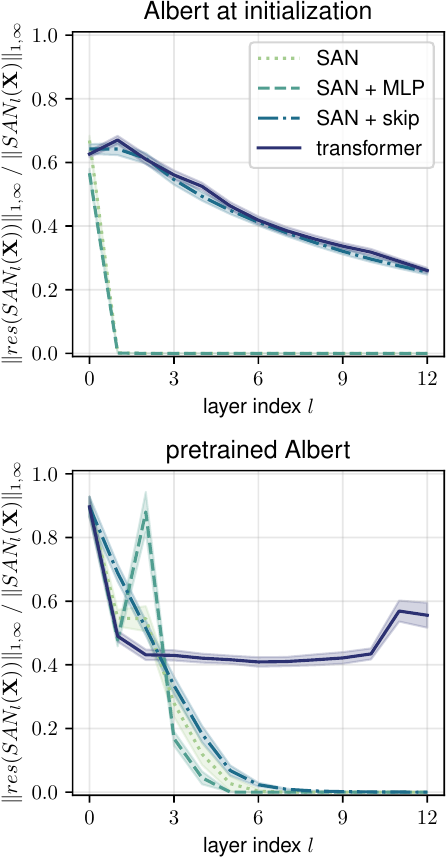}
\caption*{(b) Albert}
\end{minipage}
\hspace{3mm}
\begin{minipage}[b]{0.3\linewidth}
\centering
\includegraphics[width=1\linewidth]{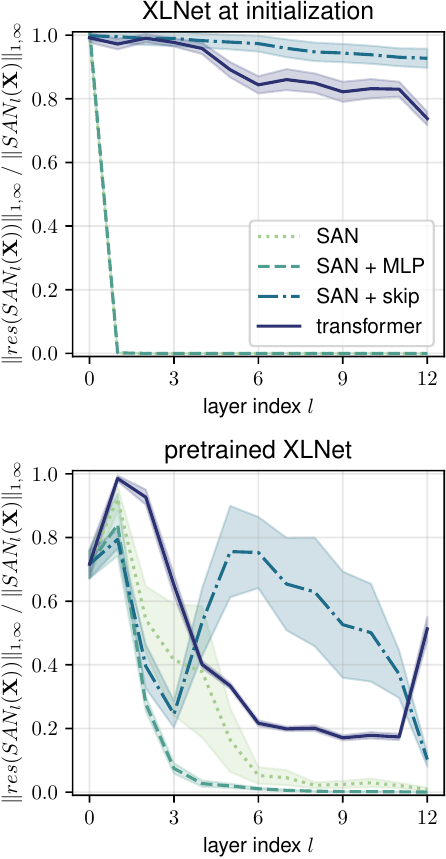}
\caption*{(c) XLNet}
\end{minipage}
\caption{Relative norm of the residual along the depth for three models before and after training. %
Pure attention (\attnet) converges rapidly to a rank-1 matrix. Adding MLP blocks and skip connection gives a transformer. Skip connections play a critical role in mitigating rank collapse (i.e., a zero residual). %
}
\label{fig:residuals}
\end{figure*}

\subsection{Rank collapse in real architectures}
\label{sec-exp-collapse}

To verify our theoretical predictions, we examine the residual of three well-known transformer architectures: BERT~\cite{devlin2018bert}, Albert~\cite{lan2019albert}, and XLNet~\cite{yang2019xlnet}. Figure~\ref{fig:residuals} plots the relative residual
$
    {\| \res(\attnet(\bX^l)\|_{1,\infty} }/{ \| \attnet(\bX^l)\|_{1,\infty} },
$
of each layer's output before and after the networks have been trained. To compute these ratios we ran the network on 32 samples of 128 tokens excerpts of biographies from Wikipedia~\citep{wikibio} and display the mean and standard deviation.

The experiments confirm that, as soon as the skip connections are removed, all networks exhibit a rapid rank collapse. Though MLPs do not seem to help in the mitigation of convergence, we caution that the observation is not an accurate portrayal of how trained transformers behave: removing the skip connections introduces a drastic distribution shift in the MLP input. We expect that the convergence will slow down if the network is retrained.

\subsection{Visualizing the bias of different architectures}
\label{sec-exp-circle}

To empirically investigate the inductive bias of the different components of the transformer architecture, we study the behavior of a single-layer transformer when applied \textit{recurrently} (akin to the universal transformer \cite{universalTransformer}) to predict a simple 2D circular sequence.

Specifically, we train a single-layer transformer to sequentially predict two circular arcs in $\mathbb{R}^2$ of radius $0.3$, starting at $(-0.3, 0)$ and $(0.3, 0)$,  respectively, each directed counter-clockwise and consisting of $1000$ points (illustrated as gray trajectories). %
An input sample consists of a sequence of two opposing points on the circle, one from the top arc and the other from the bottom arc.
We apply teacher-forcing at each step, meaning we give the network the ground truth coordinates of the two current points, and train it to predict the next two points. The model attempts to minimize the MSE loss between the predicted points and the ground truth points on the trajectories. At inference time, we don't apply teacher-forcing, and simply feed the model output as input for the next step.

\begin{figure}[t]
         \centering         \includegraphics[trim={.1cm .1cm .1cm .1cm},clip,width=.80\linewidth]{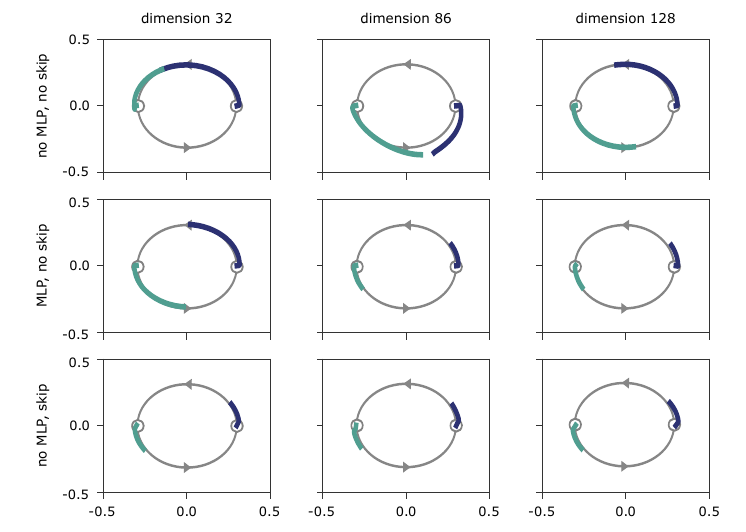}
         \caption{Applying a trained single-layer transformer module recurrently, to models of increasing hidden dimension (horizontal direction) and across architectural variants (vertical direction). The two light background paths illustrate the two training trajectories, for which the starting points are $(-0.3, 0)$ and $(0.3, 0)$. Each figure contains the same number of steps. Consistent with the theory in \S\ref{sec-counter-conv}, convergence slows down or stops as the dimension increases (since $\beta \ge \|\bW_{QK}^l\|_1 \|\bW_{V}^{l}\|_{1,\infty}$ is generally larger), as well as when either MLP or skip connections are added.
} \label{fig-circle}
\end{figure}

Since this recurrent application of a single-layer transformer can be reparametrized to be equivalent to a multi-layer transformer without skip connections, \textit{we hypothesize that at inference time the predicted trajectories of the two arcs will converge to the same point (indicating a rank collapse)}, rather than following the training trajectories. Note that the setting has also been intentionally constructed to enable training even without skip connections (by using teacher forcing) and thus to disentangle the two distinct benefits of skip connections: their ability to improve optimization and their mitigation of rank collapse. %

We trained the network until it could perfectly memorize the next step on the circular trajectories with near-zero loss. Figure~\ref{fig-circle} demonstrates the trajectories predicted at inference time (i.e., without teacher forcing). As seen on the top row, without MLP or skip connections the network exhibits rank collapse. Theorem~\ref{thm-cubic-conv} predicts that the convergence is slower when $\beta \ge \|\bW_{QK}^l\|_1 \|\bW_{V}^{l}\|_{1,\infty}$ increases. Indeed, as the hidden dimension increases from 32 to 128 (leading to larger $\beta$ at initialization), the convergence slows down, becoming hardly observable for dimension 128.

We conclude that, in accordance to our analysis, adding MLP or skip connections either stops or drastically slows down rank collapse. As observed, skip connections tend to slow down points from moving. The latter phenomenon is because in this setting skip connections introduce a bias towards remaining in the same position. On the other hand, adding MLPs does not exhibit the same bias.

\subsection{Path effectiveness}
\label{ssec:exp-paths}

\attnets can be seen as ensembles of paths of different lengths (from 0 to $L$), each involving a different sequence of self-attention heads. Our analysis of $\attnet$ with skip connections indicates that path expressivity decreases with path length, even if the number of non-linear operations involved increases.
To test this hypothesis, we isolate paths of different lengths and evaluate their predictive power.

\textbf{Tasks.} We considered the following three tasks to test path effectiveness with respect to length:
\begin{itemize}
    \item \emph{Sequence memorization.}
To solve this task, a model needs to memorize a pre-determined mapping from natural language sentences and random label sequences of the same length. %
We use random tokens (rather than actual labels) to make this purely a test of \textit{expressiveness} of a network by way of \textit{memorizing} training data, rather than confounding effects such as generalizability.
The models tested %
are trained to minimize the cross entropy loss between predicted and the ground truth labels.
The training data consist of 500 English sentences from Wikipedia and News sources \citep{pascalChallenge,glue}, which are tokenized using the SentencePiece tokenizer \citep{sentencePiece} into a vocabulary of size $30522$ with 128 tokens per sequence. Each sequence is mapped to a random binary sequence of the same length.
\item \emph{Learning to sort}.
Given an input sequence of letters, this task learns to sort the letters in alphabetical ordering (similar task have been studied before~\citep{NEURIPS2019_9001ca42}).
Specifically, the model's output for each input letter is used to determine the position of that letter in the predicted ordering. Each input sequence, of length $8$, is created by sampling uniformly randomly, with replacement, from an alphabet of size $10$. The training and test sets consist of 1000 and 200 sequences, respectively.
To ensure robustness with respect to hyperparameters, we experimented with a variety of settings (adjusting the model depth, number of heads, and the difficulty of the task by changing the alphabet size and sequence length) and observed consistent behavior. %
\item \emph{Convex hull prediction}.
This task was inspired by the work of~\cite{vinyals2015pointer}. Given a sequence of $N$ points uniformly distributed in $[0, 1]\times [0,1]$ and shifted by a random bivariate standard normal, this task predicts the convex hull of these points. Specifically, for each point in the set, the model predicts whether it's part of the convex hull.
The training set consists of $10,000$ sequences of points in $[0, 1]\times [0,1]$, each of length $10$.
\end{itemize}

In all three tasks, we report the test-set per-token label prediction accuracy as the evaluation metric.

\begin{figure*}[t]
\centering
\begin{minipage}[b]{.32\textwidth}\centering
\begin{subfigure}
\centering
     \includegraphics[trim={.2cm .2cm .5cm .4cm},clip,width=1\linewidth]{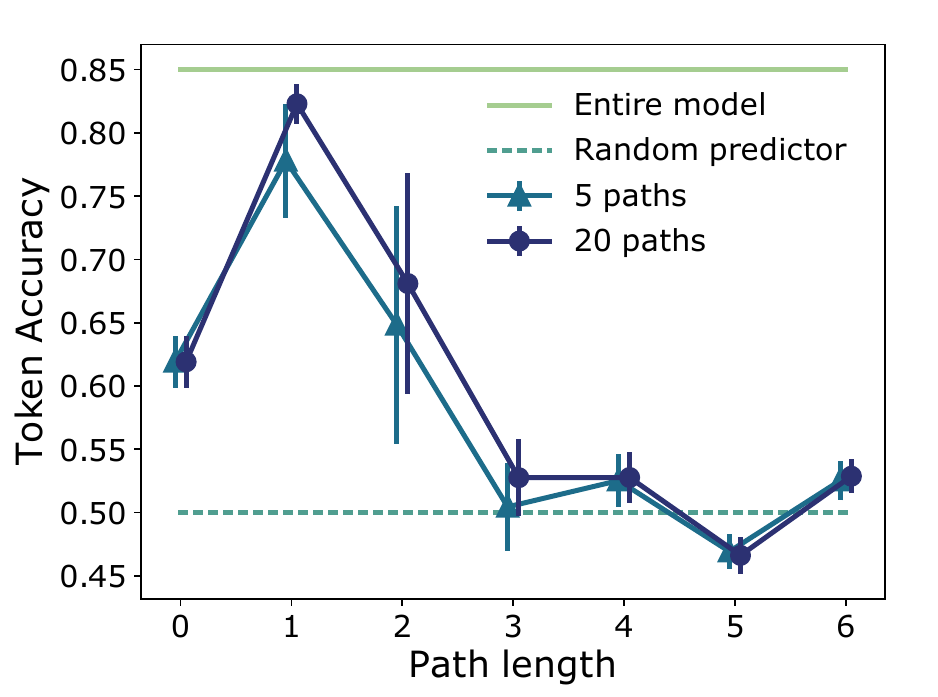}
    \caption*{(a) Memorization }
\end{subfigure}
\end{minipage}
~
\begin{minipage}[b]{.32\textwidth}\centering
\begin{subfigure}
         \centering
         \includegraphics[trim={.2cm .2cm .5cm .4cm},clip,width=1\linewidth]{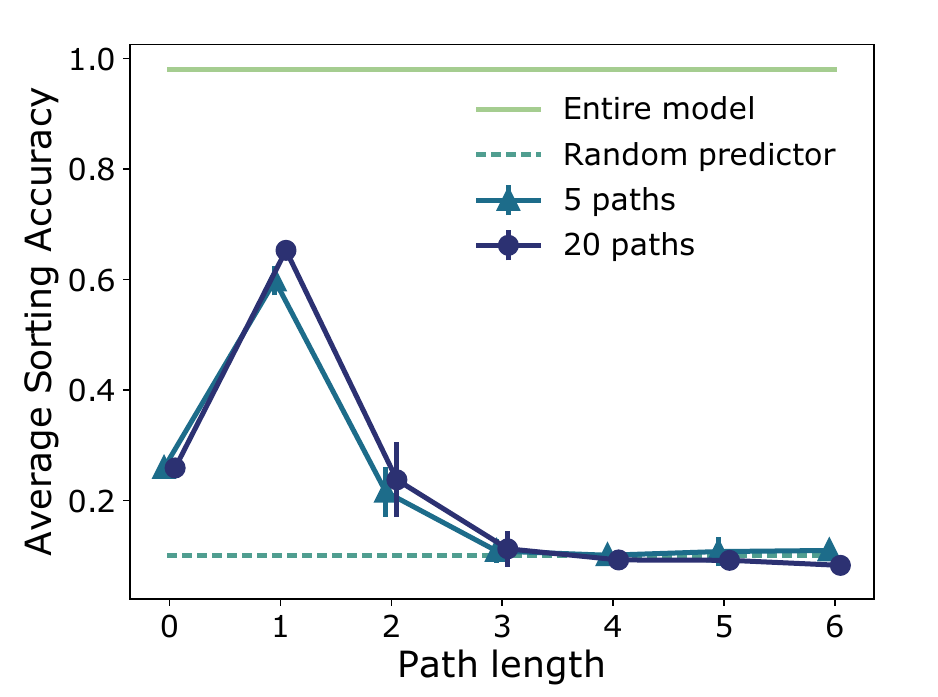}
    \caption*{(b) Sorting}
     \end{subfigure}
\end{minipage}
~
\begin{minipage}[b]{.32\textwidth}\centering
\begin{subfigure}
         \centering
         \includegraphics[trim={.2cm .2cm .5cm .4cm},clip,width=1\linewidth]{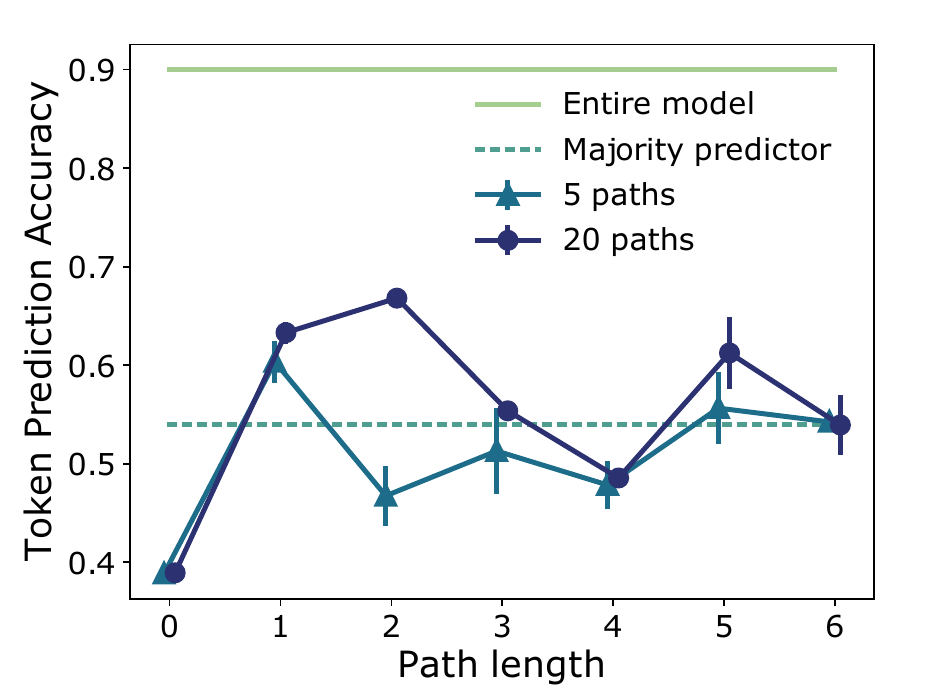}
    \caption*{(c) Convex hull}
\end{subfigure}
\end{minipage}
\caption{To determine how much of the expressive power can be attributed to short vs long paths, we examine the performance of subsets of paths of different lengths (rather than of the entire \attnet).
Performance can be seen to consistently deteriorate with respect to path length, supporting our hypothesis that short paths are responsible for the majority of the expressive power. } \label{fig-path_effectiveness}
\vspace{-10pt}
\end{figure*}

\paragraph{Path effectiveness test.} We measure the effectiveness of individual paths by a `path disentanglement' procedure that we apply at inference time: the procedure isolates the weights involved and the output of an individual path $( \bP_{h_L}^{L} \cdots \bP_{h_1}^{1}) \, \bX \, (\bW_{h_1}^{1} \cdots \bW_{h_L}^{L})$ for \textit{any} given sequence of heads $h_1, \cdots, h_L \in [H\cup 0]^L$. After the transformer has been successfully trained to solve each task (without modifications), we use this procedure to determine the output of a randomly sampled set of paths of a given length. We then evaluate the task performance based solely on the normalized sum of this subset of paths (rather than from all paths). %
Note that the training remains unaltered and uses all heads simultaneously, therefore ensuring that each path learns to its full effectiveness.

Figure~\ref{fig-path_effectiveness} illustrates the resulting performance across all three tasks. We test different subset sizes  and report the average and standard deviation of five repetitions. For reference, we also plot the accuracy of a naive classifier as well as of the entire trained model (i.e., before the path decomposition). %
As observed, short paths carry predictive power, with length-1 paths attaining accuracy above 0.8,0.6, and, 0.65 in the memorization, sorting, and convex hull tasks, respectively. On the other hand, the output of longer paths is not much better than a random guess (red horizontal lines). We note that, since there is a class imbalance in the convex hull task, we use a majority class predictor to obtain a random baseline. %
Though the difference in accuracy between short and long paths is less pronounced for the convex hull task, we observe that the variance of the long paths is significantly larger, rendering them not much better than a random guess. Length zero paths attain very small variance, but contain no useful information about the task (likely because they do not exploit global information).

The depths ($L$), number of heads ($H$), and hidden dimensions ($d$) for the three models are: $L$:6, $H$:2, $d:$250 for memorization, $L$:6, $H$:2, $d$:48 for sorting, and $L$:6, $H$:3, $d$:84 for convex hull.
It's important to note that for all three tasks, while higher \textit{peak} accuracies are attainable with increased model capacity and training time, our focus is to study the effects of path length on performance. Indeed, the trend for degenerating performance as path length increases stayed consistent across model sizes in all experiments.

The rapidly diminishing effectiveness of paths with respect to length indicates that the transformer relies almost exclusively on short paths. In other words, the transformer behaves like an ensemble of \textit{shallow} networks. Furthermore, the results indicate that there is \textit{underutilized} capacity in long paths, and suggest that one way to make them, and hence the transformer, more effective, is to prevent the long paths from losing rank.

\section{Related works}

Skip connections were first introduced in ResNets \citep{he2016deep}, ever since, it has been used to facilitate optimization in deep networks \citep{he2016identity, resnetEnsemble, balduzzi2018shattered}. In particular, skip connections tackle the vanishing gradient problem, by allowing the gradient to flow bypass the skipped layers during backpropagation.
The original motivation of using skip connections in transformers follow the same reasoning on facilitating optimization \cite{vaswaniTransformer}.
With the paths decomposition for transformers, we discover an additional surprising importance of skip connections: they prevent the transformer output from degenerating to rank one exponentially quickly with respect to network depth.

Veit et al. (\citep{resnetEnsemble}) introduced an analogous interpretation for residual networks as a collection of paths of varying lengths, and found that the length of the effective paths in deep residual networks are much shorter than the total network depth, due to the gradients used for parameter updates coming overwhelmingly from these short paths. Our finding suggests that \attnets rely on short paths to avoid rank collapse.
On the other hand, Daneshmand et al.~\citep{daneshmand2020batch} studied rank collapse in randomly initialized linear and ReLU networks and showed that batch normalization is an effective mitigation strategy.

Some recent works have approximated the attention matrix with low-rank factorizations \cite{linformerWang, synthesizerTay} or kernel methods \cite{kernelTransformer, kernelTransformerChoromanski}, to reduce the quadratic self-attention complexity. Our work is orthogonal to these works, by studying the rank of the network's output (rather than of the attention matrix).

There have been other recent advances in understanding the theory behind transformers:
\cite{perez2018on,universalTransformer} proved Turing universality, \cite{Cordonnier2020On} provided necessary and sufficient conditions for attention to simulate convolution. A linearized form of self-attention was also found to exhibit a depth phase transition \cite{levine2020limits}; and the Lipschitz constant of self-attention was analyzed by \cite{kim2020lipschitz}.

Perhaps the convergence to rank one of a path should come as no surprise: each path component contains row-stochastic matrices as a result of the softmax attention, and \cite{matrixConvergence} showed the exponential convergence of products of stochastic matrices to rank one. While the intuition behind stochastic matrices driving convergence still applies, in deep attention networks these matrices interact in more complex ways than what classical analyses consider. %
As we show, because of these interactions the rank collapses much faster than what would be expected based on classical analyses (cubic vs linear rate).

\section{Conclusion}
This work exposes competing forces over rank collapse in self-attention networks, namely self-attention vs skip connections and MLPs. In the process, we develop a path decomposition for \attnets, which modularizes the study of self-attention and is of independent interest to additional applications. These results open the door for many exciting future directions. For instance, how can one leverage the token-uniformity inductive bias revealed to design more effective networks, perhaps better at utilizing long paths? What are some practical implications for width-depth trade-off? How do we prove meaningful lower bounds of residue convergence for transformers? Answering these questions has broad implications in advancing the state of the art in deep learning.

\textbf{Acknowledgements.} Andreas Loukas would like to thank the Swiss National Science Foundation for supporting him in the context of the project ``Deep Learning for Graph-Structured Data'' (grant number PZ00P2 179981). Jean-Baptiste Cordonnier is supported by the Swiss Data Science Center (SDSC).

\bibliography{main}
\bibliographystyle{alpha}

\appendix

\section{Deferred Proofs}
\label{sec:res-bounds}

We build our argument step by step, by first considering a single-head self-attention layer in \S\ref{subsec:san-shsl} and then moving to deeper networks with single and multiple heads in \S\ref{subsec:san-shml} and~\S\ref{subsec:san-mhml}. The results are extended to take into account skip connections and MLPs in \S\ref{subsec:san-skip} and~\S\ref{subsec:san-mlp}

\subsection{Single-layer and single-head}
\label{subsec:san-shsl}

We consider a single-head self-attention layer:
$$
    \bX' = \att(\bX) = \bP \bX \bW_{V}
$$
We focus in particular on how the residual changes. As discussed previously, the value bias can be safely ignored since it does not contribute to the residual.

The following is proved:

\begin{lemma} \label{lemma-1layer-1head}
The residual abides to:
$$
\| \res(\att(\bX)) \|_{1,\infty} \leq  \frac{4 \gamma \, \|\bW_{QK}\|_1 \, \|\bW_{V}\|_{1,\infty}}{\sqrt{\dqk}}  \, \|\res(\bX)\|_{1,\infty}^3,
$$
with $\gamma$ selected such that
$\sqrt{\max_{i,j,j'} |A_{ij} - A_{ij'} | \, \sum_{i} \max_{j,j'} |A_{ij} - A_{ij'} |} \leq \gamma \max_{j,j'} \sum_{i} |A_{ij} - A_{ij'} |$ and $|E_{ij} - E_{ij'}| \leq 1.256$ with $\bE = \res(\bX) \frac{\bW_{QK}}{\sqrt{\dqk}} \res(\bX)^\T $.
\end{lemma}

The unscaled attention scores are computed as follows,
\begin{align}
    \bA = (\bX \bW_Q + \one \bb_Q^\top)(\bX \bW_K + \one \bb_K^\top)^\top
\end{align}
and following~\cite{jbCollaborate}, we can use the softmax shift invariance property to prune the terms constant over the columns and obtain,
\begin{align}
    \bA = \bX \bW_{QK} \bX^\top + \one \bb_{QK}^\top \bX^\top
\end{align}
with $\bW_{QK} = \bW_Q \bW_K^\top $ and $\bb_{QK} = \bW_K \bb_Q$.

We use the shorthand notation $\bR := \res(\bX)$ and $\bR' := \res(\bX')$.

The attention matrix can be written as
\begin{align*}
    \bA
    &= (\one \bx^\T + \bR) \frac{\bW_{QK}}{\sqrt{\dqk}} (\one \bx^\T + \bR)^\T + \one \frac{\bb_{QK}^\top}{\sqrt{\dqk}}(\one \bx^\top + \bR)^\top \\
    &= \left(\frac{\bx^\T \bW_{QK}\bx}{\sqrt{\dqk}}  \one + \bR \frac{\bW_{QK}}{\sqrt{\dqk}} \bx  + \one \frac{\bb_{QK}^\top}{\sqrt{\dqk}} \bx \right) \one^\T + \one \bx^\T \frac{\bW_{QK}}{\sqrt{\dqk}}\bR^\T +
    \bR \frac{\bW_{QK}}{\sqrt{\dqk}} \bR^\T
    + \one \frac{\bb_{QK}^\top}{\sqrt{\dqk}} \bR^\top
\end{align*}
Using the shift-invariance property of the softmax operator, the first term above can be safely ignored since it is constant across columns.
We therefore have that
\begin{align*}
    \bP
    &= \softmax \left(\bR \frac{\bW_{QK}}{\sqrt{\dqk}} \bR^\T + \one \br^\T \right),
\end{align*}
where we have set
$\br := \bR \frac{\bW_{QK}^\T}{\sqrt{\dqk}} \bx + \bR \frac{\bb_{QK}}{\sqrt{\dqk}}$.

Setting $\bE = \bR \frac{\bW_{QK}}{\sqrt{\dqk}} \bR^\T $ and $\tilde{\bA} = \one \br^\T$, the input reweighted by the attention probibilities $\bP \bX$ is given by
\begin{align}
    \bP \bX
    &= \bP (\one \bx^\T + \bR)  \\
    &= \one \bx^\T   + \bP \bR  \\
    &= \one \bx^\T  + \softmax( \one \br^\T + \bE) \bR  \\
    &\leq \one \bx^\T +  (\bI + 2 \bD) \one \, \softmax(\br)^\T \bR   \\
    &= \one (\bx^\T + \softmax(\br)^\T \bR) +  2 \bD \, \one \, \softmax(\br)^\T \bR
\end{align}
where the inequality above is entry-wise and follows from Lemma~\ref{lemma-p} whenever $|E_{ij} - E_{ij'}| \leq 1.256$.
Similarly $ \bP \bX \geq \one (\bx^\T + \softmax(\br)^\T \bR) - 2\bD \, \one \, \softmax(\br)^\T \bR $, where we again invoke Lemma~\ref{lemma-p}.

Therefore, the (entry-wise) distance of the output of the self-attention layer $SA(\bX) = \bP \bX \bW_V$ from being constant across tokens is at most:
\begin{align}
    | [SA(\bX) - \one (\br')^\T]_{ij} |
    &\leq  2 \, |[\bD \, \one \, \softmax(\br)^\T \bR \bW_V]_{ij}|,
    \label{eq:slsh_res_bound}
\end{align}
where $\br' =  (\bx + \bR^\T \softmax(\br))\bW_V$.

Now we bound the right hand side of the above inequality. For the $\ell_1$ norm we obtain:
%
\begin{align} \label{eq:slsh_softmax_bound}
    \|\bD \, \one \, \softmax(\br)^\T \bR \bW_V\|_1
    &\leq \|\bD \one\|_1\, \|\bR\|_1 \|\bW_V\|_1,
\end{align}
where the last step is due to $\|\softmax(\br)\|_1=1$ and $\|\bA\bB\|_1\le \|\bA\|_1 \|\bB\|_1$, 
implying
$ \|SA(\bX) - \one (\br')^\T\|_1 \leq 2 \|\bD \one\|_1\, \|\bR\|_1 \|\bW_V\|_1$.

On the other hand, an analogous argument gives the following bound on the $\ell_\infty$ norm of the residual:
\begin{align*}
    \|SA(\bX) - \one (\br')^\T\|_\infty
    &\leq 2 \|\bD \, \one\|_\infty \|\softmax(\br)^\T \bR \bW_V \|_\infty \\
    &\leq 2 \|\bD \, \one\|_\infty \|\bR\|_\infty \|\bW_V \|_\infty
\end{align*}
Combining the two norms we obtain:
\begin{align*}
    \|\bR'\|_{1,\infty} = \sqrt{\|\bR'\|_1 \|\bR'\|_\infty}
    &\leq 2 \sqrt{\|\bD \one\|_1 \|\bD \, \one\|_\infty } \, \|\bR\|_{1,\infty} \|\bW_V\|_{1,\infty}
\end{align*}
Moreover, by the definition of $\bD$ as in Lemma~\ref{lemma-p} and under the current Lemma's definition, we have that 
\begin{align*}
    \|\bD \one\|_1 \|\bD \, \one\|_\infty 
    &= \max_{i,j,j'} |E_{ij} - E_{ij'} | \, \sum_{i} \max_{j,j'} |E_{ij} - E_{ij'} | \\
    &= \max_{i,j,j'} |A_{ij} - A_{ij'} | \, \sum_{i} \max_{j,j'} |A_{ij} - A_{ij'} | \\    
    &\leq \gamma^2 \left(\max_{j,j'} \sum_{i} |A_{ij} - A_{ij'} |\right)^2 \tag{by assumption} \\
    &= \gamma^2 \left(\max_{j,j'} \sum_{i} |E_{ij} - E_{ij'} |\right)^2 \\
    &\leq 4 \gamma^2 \, \|\bE\|_{1}^2 \\
    &=  4 \gamma^2\, \|\bR \frac{\bW_{QK}}{\sqrt{\dqk}} \bR^\T\|_1^2 \\
    &\leq \frac{4\gamma^2}{\dqk} \, \|\bR\|_1^2 \|\bW_{QK}\|_1^2 \|\bR^\T\|_1^2 
    = \left(\frac{2\gamma}{\sqrt{\dqk}} \, \|\bR\|_{1,\infty}^2 \|\bW_{QK}\|_1\right)^2.
\end{align*}
The above imply 
\begin{align*}
    \|\bR'\|_{1,\infty} = \sqrt{\|\bR'\|_1 \|\bR'\|_\infty}
    &\leq \frac{4\gamma\, \|\bW_{QK}\|_1 \|\bW_{V}\|_{1, \infty}}{\sqrt{\dqk}}  \, \|\bR\|_{1,\infty}^3
\end{align*}
which is equivalent to the main claim.

\subsection{Multiple-heads and single-layer}
\label{subsec:san-mhsl}

\begin{lemma} \label{lemma-1layer-multi-head}
In the setting of Lemma~\ref{lemma-1layer-1head}, the residual of the output of a $H$-heads attention layer abides to:
\begin{align}
\| \res(\att(\bX)) \|_{1,\infty} \leq  \frac{4 H \gamma \beta}{\sqrt{\dqk}}  \, \|\res(\bX)\|_{1,\infty}^3\,,
\end{align}
where $\| \bW_{QK,h} \|_1 \| \bW_{h} \|_{1, \infty} \leq \beta$ for all heads $h\in[H]$.
\end{lemma}

\begin{proof}

The output of a multi-head attention layer is
\begin{align}
    SA(\bX) = \sum_{h\in [H]} \bP_h \bX \bW_{h}\, = \sum_{h\in [H]} SA_h(\bX),
\end{align}
where $\bW_{h} := \bW_{V,h} \bW_{O,h}$ as in the main text and $\bP_h$ is computed using the heads parameters $\bW_{QK,h}$ and $\bb_{QK,h}$.
The proof proceeds similarly to Section~\ref{subsec:san-mhsl} until eq.~\ref{eq:slsh_res_bound},
\begin{align}
    | [SA(\bX) - \one (\br'')^\T]_{ij} |
    &\leq  2 \, \left|\left[ \sum_h \bD_h \, \one \, \softmax(\br_h)^\T \bR \bW_{h}\right]_{ij} \right|,
\end{align}
where $\br'' =  \sum_h (\bx + \bR^\T \softmax(\br_h))\bW_{h}$.

The elementwise inequality implies inequalities for $\ell_1$ and $\ell_{\infty}$ norms and applying the triangle inequality on the sum, we obtain
\begin{align*}
    \| SA^H(\bX) - \one (\br'')^\T \|_1 &
    \leq 2 \sum_{h\in[H]} \| \bD_h \, \one \, \softmax(\br_h)^\T \bR \bW_{h} \|_1
    \leq 2H \max_{h\in[H]} \| \bD_h \, \one \, \softmax(\br_h)^\T \bR \bW_{h} \|_1
\end{align*}
and a similar expression for the  $\ell_{\infty}$ norm.
The rest of the proof proceeds similarly as the single head proof.
\end{proof}

\subsection{Single-head and multiple-layers}
\label{subsec:san-shml}

We next consider how the residual changes after $L$ layers of the form:
$
    \bX^l = \att^l_1(\bX^{l-1}).
$

\begin{customcor}{2.2}
\label{thm-single-head-rec}
In the setting of Lemma~\ref{lemma-1layer-1head}, for any single-head \attnet consisting of $L$ layers with $\|\bW_{QK,1}^l\|_1 \leq \beta$ for every $l \in [L]$, the residual is bounded by
\begin{align}
   \|\res(\attnet(\bX))\|_{1,\infty} \leq \left(\frac{4 \, \gamma \beta}{\sqrt{\dqk}}\right)^{\frac{3^L-1}{2}} \, \|\res (\bX) \|^{3^L}_{1,\infty},
\end{align}
which amounts to a doubly exponential convergence to a rank-1 matrix.
\end{customcor}

\begin{proof}
Unfolding the recursion backwards from the last layer to the first and applying Lemma~\ref{lemma-1layer-1head} we obtain:
\begin{align}
    \| \res(\bX^L)\|_{1,\infty}
    &\leq \frac{4 \, \gamma \beta}{\sqrt{\dqk}} \, \|\res(\bX^{L-1})\|_{1,\infty}^3 \\
    &\leq \frac{4 \, \gamma \beta}{\sqrt{\dqk}} \, \left( \frac{4 \, \gamma \beta}{\sqrt{\dqk}} \, \|\res(\bX^{L-2})\|_{1,\infty}^3\right)^3 \\
    &= \frac{4 \, \gamma \beta}{\sqrt{\dqk}} \, \left(\frac{4 \, \beta}{\sqrt{\dqk}}\right)^3 \, \|\res(\bX^{L-2})\|_{1,\infty}^{3^2} \\
    &\leq \ldots \\
    &\leq \prod_{l=1}^L \left(\frac{4 \, \gamma \beta}{\sqrt{\dqk}}\right)^{3^{l-1}} \, \|\res(\bX)\|_{1,\infty}^{3^L}
    = \left(\frac{4 \, \gamma \, \beta}{\sqrt{\dqk}}\right)^{\frac{3^L-1}{2}} \, \|\res(\bX)\|_{1,\infty}^{3^L},
\end{align}
matching the theorem statement.
\end{proof}

\subsection{Multiple-head and multiple-layers}
\label{subsec:san-mhml}

\begin{customcor}{2.3}[mutli-head multi-layer] \label{lem-conv-all-paths}
In the setting of Lemma~\ref{lemma-1layer-1head}, consider a depth-$L$ \attnet with $H$ heads per layer. Fix $\|\bW_{QK,h}^l \|_1 \| \bW_{h}^{l}\|_{1,\infty}\leq \beta$ for all $h \in [H]$ and $l \in [L]$. The output residual is bounded by
\begin{align}
    \| \res(\bX^L)\|_{1,\infty}
    &\leq \left( \frac{4 \, H \, \gamma \, \beta}{ \sqrt{\dqk}}\right)^{\frac{3^L-1}{2}} \, \|\res(\bX)\|_{1,\infty}^{3^L},
\end{align}
which indicates that the output convergences to a rank-1 matrix doubly exponentialy.
\end{customcor}
\begin{proof}
    The proof procceeds recursively as for Theorem~\ref{thm-single-head-rec} in the single head case but using the bound on single-layer multi-heads residuals from Lemma~\ref{lemma-1layer-multi-head}.
\end{proof}

\subsection{\attnet with skip connections}
\label{subsec:san-skip}
As noted in the main text, a lower bound on the residual better aligns with practice, where \attnets with skip connections do not suffer rank collapse.
For consistency with the other analyses and as one way to illustrate residual growth, we provide a (vacuously large) upper bound on the residual for \attnets with skip connections.

\begin{customcor}{3.1}
[\attnet with skip connections] \label{cor-san-with-skip}
In the setting of Lemma~\ref{lemma-1layer-1head}, consider a depth-$L$ \attnet with $H$ heads per layer and skip connections. Fix $\|\bW_{QK,h}^l\|_1 \| \bW_{h}^l \|_{1, \infty} \leq \beta$ for all heads $h \in [H]$ and layers $l \in [L]$. The output residual is bounded by
\begin{align*}
    \|\res(\bX^L)\|_{1,\infty}
    \leq \max_{0 \le l \le L} \left(\frac{8 \, \gamma \, \beta\,  H}{\sqrt{\dqk}}\right)^{\frac{3^l-1}{2}} \, (2H)^{3^l(L-l)}\|\res(\bX)\|_{1,\infty}^{3^l},
\end{align*}
which does not indicate convergence.
\end{customcor}

\begin{proof}
For a \attnet with skip connections, the residual bound for a single-head single-layer \attnet from lemma~\ref{lemma-1layer-1head} now becomes:
\begin{align}\label{eq-bound-single-head-layer-skip}
    \|\res(\attnet(\bX))\|_{1,\infty} \le \frac{4\, \gamma \|\bW_{QK,h}\|_1 \|\bW_{V}\|_{1, \infty}}{\sqrt{\dqk}}  \, \|\res(\bX)\|_{1,\infty}^3 + \|\res(\bX)\|_{1,\infty}
\end{align}

To obtain a multi-layer bound, we unfold the recursion backwards.

Let us consider a single head model first and fix $\|\bW_{QK,h}^l\|_1 \| \bW_{h}^l \|_{1, \infty} \leq \beta$ for all $l \in [L]$. We have that:
\begin{align} %
    \| \res(\bX^L)\|_{1,\infty}
    &\leq \frac{4 \, \gamma \, \beta}{\sqrt{\dqk}} \, \|\res(\bX^{L-1})\|_{1,\infty}^3 + \|\res(\bX^{L-1})\|_{1,\infty} \nonumber \\
    &\leq 2\max(\frac{4 \, \gamma \, \beta}{\sqrt{\dqk}} \, \|\res(\bX^{L-1})\|_{1,\infty}^3, \,  \|\res(\bX^{L-1})\|_{1,\infty} ) \label{eq:res-skip-2-sides}
\end{align}
Now we unroll this bound across layers to write it in terms of $\res(\bX)$. At the $k^{th}$ step of unrolling, the max is one of the two terms in Eq~\ref{eq:res-skip-2-sides}: either $\frac{4 \, \gamma \, \beta}{\sqrt{\dqk}} \, \|\res(\bX^{L-k})\|_{1,\infty}^3$ or $\|\res(\bX^{L-k})\|_{1,\infty}$, i.e. we make a binary choice. Thus unrolling through all $L$ layers corresponds to a path from the root to the maximum leaf in a depth-$L$ complete binary tree. Each leaf has the form
$ \left(\frac{8 \, \gamma \, \beta}{\sqrt{\dqk}}\right)^{\frac{3^l-1}{2}} \, 2^{3^l(L-l)}\|\res(\bX)\|_{1,\infty}^{3^l}$, where $l$ indicates the number of times the term $\frac{4 \, \gamma \, \beta}{\sqrt{\dqk}} \, \|\res(\bX^{L-k})\|_{1,\infty}^3$ is chosen as the max. Note the ordering of these choices does not matter, only the number of times a term is chosen. Consequently, the residual bound is the maximum amongst such leaf terms:
\begin{align*}
    \| \res(\bX^L)\|_{1,\infty} \leq \max_{0 \le l \le L} \left(\frac{8 \, \gamma \, \beta}{\sqrt{\dqk}}\right)^{\frac{3^l-1}{2}} \, 2^{3^l(L-l)}\|\res(\bX)\|_{1,\infty}^{3^l}.
\end{align*}

We now apply this bound to $H$ heads, we use Lemma~\ref{lemma-1layer-multi-head}, which for a single layer gives:
\begin{align*}
    \|\res(\attnet(\bX))\|_{1,\infty} \le \frac{4\, \gamma \, \beta \, H}{\sqrt{\dqk}}  \, \|\res(\bX)\|_{1,\infty}^3 + H \|\res(\bX)\|_{1,\infty}
\end{align*}
Therefore, accounting for the factor of $H$ in above, we obtain a residual bound for a depth-$L$ width-$H$ $\attnet$ with  skip connections:
\begin{align*}
    \|\res(\bX^L)\|_{1,\infty}
    \leq \max_{0 \le l \le L} \left(\frac{8 \, \gamma \, \beta\,  H}{\sqrt{\dqk}}\right)^{\frac{3^l-1}{2}} \, (2H)^{3^l(L-l)}\|\res(\bX)\|_{1,\infty}^{3^l},
\end{align*}
which concludes the proof.
\end{proof}

\subsection{\attnet with MLP}
\label{subsec:san-mlp}

We now study how using an MLP affects the residual.
Recall we focus on \attnets with layers written as
\begin{align}
    \bX^{l+1}
    &= f_l\left( \sum_{h \in [H]} \bP_h \bX^{l} \bW_{h}\right).
    \label{eq:mlp-formulation}
\end{align}
Note that, to keep the notation compact, we use $f_l$ to encompass both the MLP as well as the output bias.

In our subsequent analysis, we use $\lambda_{l, 1,\infty}$ to denote the Lipschitz constant of $f_l$ with respect to $\ell_{1,\infty}$ norm.

The proof proceeds the same way as in \S\ref{subsec:san-shsl}.
For clarity, we point out the differences with proof in \S\ref{subsec:san-shsl} without repeating details that remain the same.

\begin{customthm}{3.2}
[\attnet with MLP]
In the setting of Lemma~\ref{lemma-1layer-1head}, consider a depth-$L$ and width-$H$ \attnet with MLP. Moreover, let $\|\bW_{QK,h}^l \|_1 \| \bW_{h}^{l}\|_{1,\infty}\leq \beta$ for all $h \in [H]$ and $l \in [L]$ and fix $\lambda_{l, 1,\infty} \leq \lambda$. We then have that
\begin{align}
    \| \res(\bX^L)\|_{1,\infty}
    &\leq \left( \frac{4 \, \gamma \, \beta \, H\, \lambda} { \sqrt{\dqk}}\right)^{\frac{3^L-1}{2}} \, \|\res(\bX)\|_{1,\infty}^{3^L},
\end{align}
which amounts to a doubly exponential rate of convergence. %
with respect to the $\ell_{1,\infty}$ norm.
\end{customthm}

\begin{proof}
With an MLP as formulated in Eq~\ref{eq:mlp-formulation}, we have $\bW_h :=\bW_V \bW_O$ in place of just the value weight $\bW_V$, as defined in the main text. As before, let $\bR$ denote $\res(\bX)$.

The proof proceeds the same way as in Lemma~\ref{lemma-1layer-1head}, until Eq~\ref{eq:slsh_res_bound},
where we handle the multi-head case the same way as in Eq~\ref{lemma-1layer-multi-head} to obtain the entrywise inequality:
\begin{align}
    \left| \left[ \sum_{h \in [H]} \bP_h \bX^{l} \bW_{h} - \one (\br')^\T\right]_{ij} \right|
    &\leq  2 \, \left|\left[ \sum_h \bD_h \, \one \, \softmax(\br_h)^\T \bR \bW_{h}\right]_{ij} \right|,
    \label{eq:mlp-res-softmax-bound}
\end{align}
As in the proof of \ref{lemma-1layer-multi-head}, this elementwise inequality implies the corresponding inequality in matrix norms $\ell_1$ and $\ell_\infty$, to each of which we apply the triangle inequality to yield:
\begin{align*}
    \left\Vert \sum_{h \in [H]} \bP_h \bX^{l} \bW_{h} - \one (\br')^\T \right\Vert_p &
    \leq 2H \max_{h\in[H]} \| \bD_h \, \one \, \softmax(\br_h)^\T \bR \bW_{h} \|_p,
\end{align*}
for $p \in [1,\infty]$.

We now use the fact that $f(\one r'^\top)$ also takes the form $ \one r''^\top$ for some vector $r''$. Indeed, $f$ encompasses weight matrix multiplications, bias addition, and entrywise nonlinearities, all of which preserve the fact that $f(\one r'^\top)$ is constant across rows. Therefore, \begin{align*}\|\res(\attnet(\bX) )\|_p &=
\left\Vert f\left(\sum_{h \in [H]} \bP_h \bX^{l} \bW_{h}\right) - \one r''^\T\right\Vert_p &\\
&=
\left\Vert f\left(\sum_{h \in [H]} \bP_h \bX^{l} \bW_{h}\right) - f(\one (r')^\T) \right\Vert_p &\rhd\, \text{$f$ preserves constancy-across-rows.} \\
&\leq \lambda_{l,p} \left\Vert \sum_{h \in [H]} \bP_h \bX^{l} \bW_{h} - \one (r')^\T\right\Vert_p
&\rhd\, \text{By definition of Lipschitz constant.} \\
&\leq 2\lambda_{l,p}\, H \max_{h\in[H]} \| \bD_h \, \one \, \softmax(\br_h)^\T \bR \bW_{h} \|_p &\rhd\, \text{ By Eq~\ref{eq:mlp-res-softmax-bound}.}
\end{align*}
Subsequently, just like for the single-head single-layer proof, we bound $\| \bD_h \, \one \, \softmax(\br_h)^\T \bR \bW_{h} \|_p$ in the above by
\begin{align}
    \|\bD_h \, \one \, \softmax(\br_h)^\T \bR \bW_{h} \|_1
    &\leq \|\bD_h \one\|_1 \, \|\bR\|_1 \|\bW_h\|_1 , \label{eq:res-bound-l1}\\
    \|\bD_h \, \one \, \softmax(\br_h)^\T \bR \bW_{h} \|_\infty
    &\leq \|\bD_h \one\|_\infty \, \|\bR\|_\infty \|\bW_h\|_\infty.
    \label{eq:res-bound-linfty}
\end{align}

As we have shown before, $\|\bD_h\one\|_{1,\infty}$ can be bounded above by $\frac{2\gamma \, }{\sqrt{\dqk}} \, \|\bR\|_1 \|\bW_{QK,h}\|_1 \|\bR\|_\infty$. Applying this to both Eq~\ref{eq:res-bound-l1} and Eq~\ref{eq:res-bound-linfty}, and combining the two as in Lemma~\ref{lemma-1layer-1head}, yields the bound:
\begin{align*}
\|\res(\attnet(\bX )) \|_{1,\infty}
\leq
\frac{4\, \gamma \, H\, \lambda_{l, 1,\infty} \|\bW_{QK,h}\|_1 \|\bW_{V}\|_{1, \infty}}{\sqrt{\dqk}}  \, \|\res(\bX)\|_{1,\infty}^3
\end{align*}
Finally, we recursively unroll the bound across layers to obtain a residual bound in terms of $\res(\bX)$:
\begin{align*}
    \| \res(\bX^L)\|_{1,\infty}
    &\leq \left( \frac{4 \, \gamma \, \beta\, H\, \lambda_{l, 1,\infty}} { \sqrt{\dqk}}\right)^{\frac{3^L-1}{2}} \, \|\res(\bX)\|_{1,\infty}^{3^L},
\end{align*}
which concludes the proof.
\end{proof}

\subsection{A technical lemma}
\begin{lemma}\label{lemma-p}
Suppose that $\bP$ is the row-stochastic matrix associated with $\bA$ and let $\tilde{\bP}$ be the one associated with $ \tilde{\bA} = \bA - \bE$ for some matrix $\bE$ with $|E_{ij} - E_{ij'}| \leq 1.256$ for every $i,j,j'$. Then
$$
(\bI - \bD) \, \tilde{\bP} \leq \bP \leq (\bI + 2 \bD) \, \tilde{\bP}
$$
with the diagonal matrix $\bD$ having $D_{ii} = \max_{j,j'} |\bdelta_i^\T \bE (\bdelta_j - \bdelta_{j'}) |$ and the inequality taken entry-wise.
\end{lemma}

\begin{proof}
Let us start by the definition of the row-stochastic matrix:
\begin{align*}
    P_{ij}
    = [\softmax(\bA)]_{ij}
    = [\softmax(\tilde{\bA} + \bE)]_{ij}
    = \frac{\exp{(\tilde{A}_{ij} + E_{ij})}}{ \sum_{t = 1}^n \exp{(\tilde{A}_{it} + E_{it})} }
    = \frac{\exp{(\tilde{A}_{ij})} \, \exp{(E_{ij})}}{ \sum_{t = 1}^n \exp{(\tilde{A}_{it})} \, \exp{(E_{it})} }
\end{align*}
The above, implies that for every $i,j$ we have:
$$
\min_{j'} \exp{(E_{ij} - E_{ij'})} \, \tilde{P}_{ij} \leq P_{ij} \leq \tilde{P}_{ij} \, \max_{j'} \exp{(E_{ij} - E_{ij'})},
$$
which can be further relaxed to
$$
(1 - 2 \min_{j'} (E_{ij} - E_{ij'})) \, \tilde{P}_{ij} \leq P_{ij} \leq \tilde{P}_{ij} \, ( 1 + 2 \max_{j'} (E_{ij} - E_{ij'})).
$$
which holds for $|E_{ij} - E_{ij'}| \leq 1.256$.
Notice also that
$$
    \max_{j'} (E_{ij} - E_{ij'}) = \max_{j} \bdelta_i^\T \bE (\bdelta_j - \bdelta_{j'})
\quad \text{and} \quad
    \min_{j'} (E_{ij} - E_{ij'}) = \max_{j'} \bdelta_i^\T \bE (\bdelta_{j'}-\bdelta_j),
$$
both of which are at most $
\max_{j'} |\bdelta_i^\T \bE (\bdelta_j - \bdelta_{j'}) |$, from which the claim follows.
\end{proof}

\section{Additional results}
\label{app-exp}

\subsection{The path length distribution of transformers}
\label{sec-paths-counting}

\begin{figure*}[h]
     \includegraphics[width=.9\linewidth]{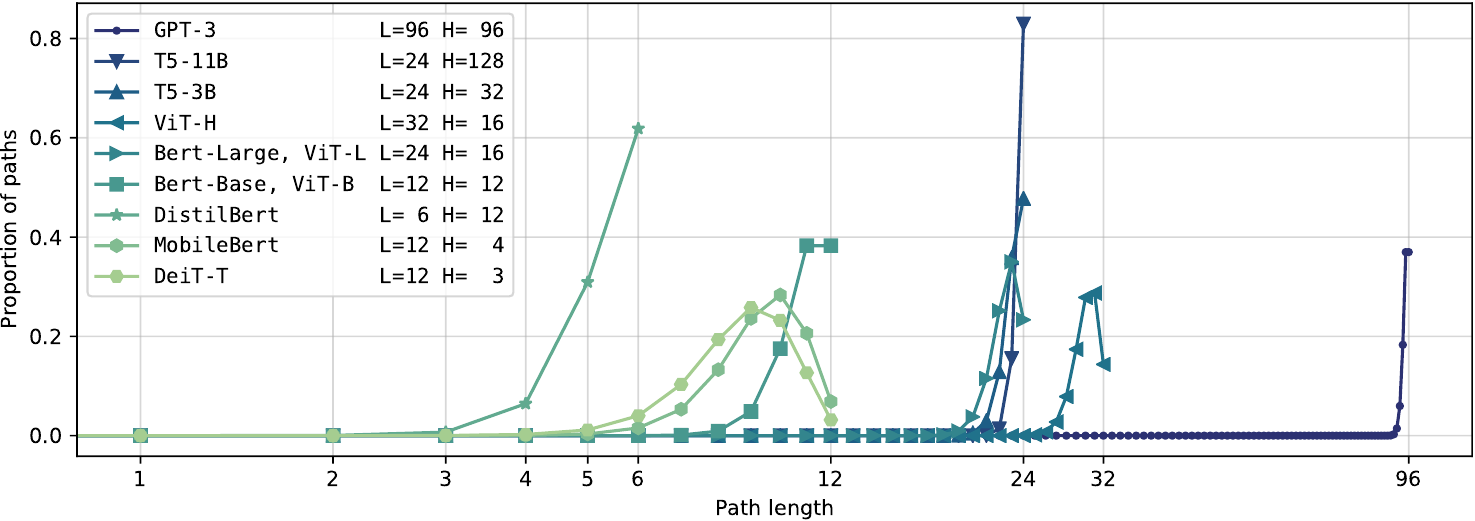}
     \caption{Distribution of the path length for a diverse selection of transformer architectures (encoder only) with different depths and widths.
     The legends are sorted by the total number of heads in the architecture L$\times$H.
     We provide the following architecture: GPT-3 \citep{gpt3}, T5 \citep{2020t5}, Bert \citep{devlin2018bert}, ViT \citep{dosovitskiy2021an}, DistilBert \citep{distilbert}, MobileBert \citep{mobilebert}.}
     \label{fig:path-distr-2}
\end{figure*} \textbf{}

As we saw in~\S2.1, transformers can be viewed as an interdependent ensemble of simpler networks (or paths) each of different depth (or length).  Aiming to gain more insight about the ensemble structure in practice, Fig~\ref{fig:path-distr-2} visualizes the path length distribution in various commonly-used architectures.

Based on the exponential decay of path effectiveness result, we hypothesize that models that focus overwhelmingly on long paths are less efficient than models with a more diverse path distribution.
The long-paths models are furthermore likely to be less robust, as they require larger MLP Lipschitz constants to counteract the token-uniformity inductive bias caused by self-attention, as described in \S3.
It is perhaps no coincidence that the intentionally more efficient models, such as DistilBert or MobileBert, have some of the most diverse path distributions; and that for the most extreme long-paths-focused model, GPT3, studies found that its model size can be reduced by several orders of magnitude and achieve similar performance~\citep{schick2020s}. We leave these exciting directions for future work.

\end{document}